\documentclass[conference]{IEEEtran}
\IEEEoverridecommandlockouts

\usepackage{cite}
\usepackage{amsmath,amssymb,amsfonts}
\usepackage{graphicx}
\usepackage{textcomp}
\usepackage{xcolor}
\usepackage{amsthm}
\usepackage{algorithmic}
\usepackage[ruled]{algorithm2e}
\usepackage{bm}
\usepackage{bbm}
\usepackage{hyperref}   
\usepackage{mathtools}
\usepackage{booktabs}
\usepackage{multirow}
\usepackage{enumitem}
\usepackage{wasysym}    
\usepackage{graphicx}   

\newtheorem{assumption}{Assumption}
\newtheorem{lemma}{Lemma}
\newtheorem{theorem}{Theorem}
\newtheorem{definition}{Definition}
\newtheorem{proposition}{Proposition}

\usepackage{pifont}
\newcommand{\cmark}{\textcolor{green}{\ding{52}}}
\newcommand{\xmark}{\textcolor{red}{\ding{55}}}
\newcommand{\middlehrulefill}{%
  \leavevmode\xleaders\hbox{\rule{0.6em}{0.4pt}}\hfill\kern0pt%
}
\newcommand{\tightemoji}[2][2.7mm]{\smash{\raisebox{#1}{#2}}}

\newcommand{\arctopK}{{\texttt{\textbf{ARC}}\nobreakdash-\texttt{\textbf{Top}}\nobreakdash-{$K$}}}
\newcommand{\RandK}{{\texttt{\textbf{Rand}}\nobreakdash-{$K$}}}
\newcommand{\TopK}{{\texttt{\textbf{Top}}\nobreakdash-{$K$}}}

\usepackage{array}
\newcolumntype{P}[1]{>{\centering\arraybackslash}p{#1}}
\newcolumntype{L}[1]{>{\left\arraybackslash}p{#1}}

\def\BibTeX{{\rm B\kern-.05em{\sc i\kern-.025em b}\kern-.08em
    T\kern-.1667em\lower.7ex\hbox{E}\kern-.125emX}}

\allowdisplaybreaks

\begin{document}
\title{An All-Reduce Compatible {\texttt{Top}\nobreakdash-{$K$}} Compressor for Communication-Efficient Distributed Learning\\
}

\author{
\IEEEauthorblockN{Chuyan Chen$^{*}$}
\IEEEauthorblockA{
\textit{Peking University}\\
Beijing, China \\
chuyanchen@stu.pku.edu.cn}
\and
\IEEEauthorblockN{Chenyang Ma$^{*}$}
\IEEEauthorblockA{
\textit{Peking University}\\
Beijing, China \\
2300010754@stu.pku.edu.cn}
\and
\IEEEauthorblockN{Zhangxin Li$^{*}$}
\IEEEauthorblockA{
\textit{Peking University}\\
Beijing, China \\
2200011085@stu.pku.edu.cn}
\and
\IEEEauthorblockN{Yutong He}
\IEEEauthorblockA{
\textit{Peking University}\\
Beijing, China \\
yutonghe@pku.edu.cn}
\and
\hspace{4cm}
\IEEEauthorblockN{Yanjie Dong$^\dagger$}
\hspace{4cm}\IEEEauthorblockA{
\hspace{4cm}\textit{Shenzhen MSU-BIT University}\\
\hspace{4cm} Shenzhen, P. R. China \\
\hspace{4cm} ydong@smbu.edu.cn}
\and
\IEEEauthorblockN{Kun Yuan$^\dagger$}
\IEEEauthorblockA{
\textit{Peking University}\\
Beijing, P. R. China \\
kunyuan@pku.edu.cn}
\thanks{$^{*}$Equal contributions.\quad$^\dagger$Co-Corresponding author}
\thanks{The work is supported by the NSF China under Grants 92370121, the National Key Research and Development Program of China under Grant 2024YFA1012902, and Guangdong Provincial Key Research Project for Regular Universities (2025ZDZX3048)}
}

\maketitle

\begin{abstract}
Communication remains a central bottleneck in large-scale distributed machine learning, and gradient sparsification has emerged as a promising strategy to alleviate this challenge. However, existing gradient compressors face notable limitations: \RandK\ discards structural information and performs poorly in practice, while \TopK\ preserves informative entries but loses the contraction property and requires costly \texttt{All-Gather} operations. In this paper, we propose \arctopK, an \texttt{All-Reduce}-Compatible Top-$K$ compressor that aligns sparsity patterns across nodes using a lightweight sketch of the gradient, enabling index-free \texttt{All-Reduce} while preserving globally significant information. \arctopK\ is provably contractive and, when combined with momentum error feedback (EF21M), achieves linear speedup and sharper convergence rates than the original EF21M under standard assumptions. Empirically, \arctopK\ matches the accuracy of \TopK\ while reducing wall-clock training time by up to 60.7\%, offering an efficient and scalable solution that combines the robustness of \RandK\ with the strong performance of \TopK.



\end{abstract}


\section{Introduction}
The rise of large-scale machine learning has established distributed training as a fundamental paradigm in modern AI systems. In centralized cloud environments, massive datasets are partitioned and sharded across $N$ nodes (e.g., GPUs) within a data center to enable data-parallel training~\cite{shi2016edge}~\cite{satyanarayanan2017emergence}. In contrast, federated learning keeps data on $N$ client devices (e.g., smartphones) for privacy and regulatory compliance, while a coordinating server aggregates model updates instead of raw data~\cite{konevcny2016federated,konevcny2016federated2,li2020federated}. Both paradigms give rise to the following distributed stochastic optimization problem:
\begin{align}
    \min_{\bm{x}\in\mathbb{R}^{d}}\ f(\bm{x}) \coloneqq \left\{\frac{1}{N}\sum_{i=1}^{N} \left[f_i(\bm{x}) \coloneqq \mathbb{E}_{\bm{\xi}\sim\mathcal{D}_i}F_i(\bm{x},\bm{\xi})\right] \right\}.
    \label{eq:formulation}
\end{align}
Here, $N$ denotes the number of computing nodes, $\bm{\xi}$ represents a random data sample drawn from the local distribution $\mathcal{D}_i$, and $f_i: \mathbb{R}^d \rightarrow \mathbb{R}$ is the local loss function at node $i$, which may be non-convex. While each node possesses its local data and loss function, they collaborate to minimize the global loss function $f(\bm{x})$ defined in problem \eqref{eq:formulation}.

In such distributed settings, each node $i$ can only access its local data and calculate local gradient $\nabla F_i(\bm{x};\bm{\xi})$ of its loss function during optimization. However, communication is required to obtain information from other nodes, which frequently becomes the primary bottleneck and dominates end-to-end training time. The associated cost scales with both the model dimension $d$ and the number of nodes $N$, rendering it prohibitive for large-scale learning tasks. Consequently, minimizing the volume of communicated bits per round is a critical problem for efficient and scalable distributed learning.

\begin{table*}[tb!]  
\renewcommand{\arraystretch}{1.5}
\begin{center}
\caption{\small Comparison between different gradient sparsification compressors. The symbol ``$\downarrow$'' indicates that a smaller value of the quantity is preferable. The symbol {\color{green}{\Large \smiley}} implies superior performance} 
\vspace{-2mm}
{\tabcolsep=2pt
\begin{tabular}{lP{2cm}P{1.8cm}P{2.6cm}P{2.8cm}P{2.2cm}P{2.2cm}}
\toprule
\textbf{\renewcommand{\arraystretch}{1.0}\begin{tabular}[c]{@{}c@{}}\hspace{3mm}EF21M with \\ \hspace{3mm}Compressor \end{tabular}} &
\textbf{\renewcommand{\arraystretch}{1.0}\begin{tabular}[c]{@{}c@{}}\hspace{-4mm}Global Contractive \\ \hspace{-3mm}Property \end{tabular}} & 
\textbf{Primitive} &
\textbf{\renewcommand{\arraystretch}{1.0}\begin{tabular}[c]{@{}c@{}}Communication \\ \hspace{0.5mm} per Iteration $(\downarrow)$ \end{tabular}} &
\textbf{\renewcommand{\arraystretch}{1.0}\begin{tabular}[c]{@{}c@{}}Asymptotic \\ Convergence Rate $(\downarrow)$ \end{tabular}} &
\textbf{\renewcommand{\arraystretch}{1.0}\begin{tabular}[c]{@{}c@{}}Transient \\ Complexity $(\downarrow)$ \end{tabular}} &
\textbf{\renewcommand{\arraystretch}{1.0}\begin{tabular}[c]{@{}c@{}}Empirical \\ Performance \end{tabular}}\\
\midrule
\texttt{\textbf{No Compression}}& \hspace{-4mm}- & \texttt{All-Reduce} & $2mn$  & $\mathcal{O}(\frac{\sigma}{\sqrt{NT}})$ & $\mathcal{O}(\frac{N}{\sigma^2})$ & {\color{green}\tightemoji{\Large \smiley \vspace{-1mm}}} \vspace{-3mm}\\
\midrule
\TopK\  & \hspace{-4mm}\xmark &\texttt{All-Gather} & $(N-1)(nK+K)$ & $\mathcal{O}(\frac{\sigma}{\sqrt{NT}})$ & $\mathcal{O}(\frac{N^3}{\sigma^2})$ & {\color{green}\tightemoji{\Large \smiley \vspace{-1mm}}} \vspace{-2.5mm}\\
\RandK\  & \hspace{-4mm}\cmark &\texttt{All-Reduce} & $2Kn$ & $\mathcal{O}(\frac{\sigma}{\sqrt{NT}})$ & $\mathcal{O}(\frac{N}{\sigma^2})$ & {\color{red}\tightemoji{\Large \frownie \vspace{-1mm}}} \vspace{-2.5mm}\\
\arctopK\  &  \hspace{-4mm}\cmark&\texttt{All-Reduce} & $2Kn +2mr$ & $\mathcal{O}(\frac{\sigma}{\sqrt{NT}})$ & $\mathcal{O}(\frac{N}{\sigma^2})$ & {\color{green}\tightemoji{\Large \smiley \vspace{-1mm}}} \vspace{-2.5mm}\\
\arctopK\ ($r=1$)& \hspace{-4mm}\cmark &\texttt{All-Reduce} & $2Kn +2m$ & $\mathcal{O}(\frac{\sigma}{\sqrt{NT}})$ & $\mathcal{O}(\frac{N}{\sigma^2})$ & {\color{green}\tightemoji{\Large \smiley \vspace{-1mm}}} \vspace{-3mm}\\
\bottomrule\end{tabular}}
\vspace{-7mm}
\label{table:communication_cmp}
\end{center}
\end{table*}

\subsection{Limitations of Existing Communication-Saving Approaches}

Gradient sparsification is a prominent strategy to reduce communication by transmitting only a small fraction of gradient coordinates. The two most common gradient sparsification compressors are \RandK\ and \TopK\ \hspace{-2mm}~\cite{wangni2018gradient}~\cite{stich2018sparsified}. \RandK\ samples $K$ coordinates of the gradient uniformly at random to form an unbiased estimate, whereas \TopK\ preserves the $K$ entries with the largest magnitudes. Despite the apparent trade-off between structure-awareness and statistical simplicity, both methods suffer from significant drawbacks.

\RandK\ employs uniform sampling that disregards gradient structure, often discarding informative coordinates and inducing substantial compression error, which in turn leads to degraded accuracy and slower convergence in practice. In contrast, \TopK\ leverages gradient magnitudes but introduces two key challenges: algorithmically, independent \TopK\ selection across nodes misaligns entry indices, breaking error contraction properties of the global gradient and undermining theoretical convergence guarantees; systemically, the absence of shared indices requires transmitting both values and indices, precluding efficient \texttt{All-Reduce} and forcing slower \texttt{Gather/Scatter} primitives that increase latency and underutilize bandwidth. These motivate the research question:
\vspace{-1mm}
\begin{center}
\setlength\fboxsep{4pt}
\colorbox{gray!20}{\parbox{\dimexpr\linewidth-2\fboxsep\relax}{
\textit{\textbf{(Question)} Can we design a communication-efficient sparsification method that not only \textbf{selects informative gradient components} but also enables the use of \textbf{index-free collectives} (e.g., \texttt{All-Reduce}), achieving both superior \textbf{convergence properties} and strong \textbf{empirical performance}?}
}}
\end{center}

\vspace{-2mm}
\subsection{Main Results}
To address the fundamental question, we propose a novel \textbf{\texttt{{\underline{A}ll-\underline{R}educe-\underline{C}ompatible Top}-$K$}} (\textbf{\texttt{ARC-Top}-$K$}) compressor in this paper. \arctopK\ has three key features: 
\begin{itemize}[leftmargin=1em]
\item \textbf{Informative gradient entry selection.} \arctopK\ adapts to gradient structure and selects statistically informative coordinates. This design ensures performance comparable to \TopK, while achieving higher accuracy and faster convergence than \RandK which uses random sparsification.

\vspace{1mm}
\item \textbf{Superior convergence guarantees.} \arctopK\ maintains the contractive property of the compressed global gradient. Consequently, communication-efficient methods using \arctopK\ as the gradient compressor can achieve faster theoretical convergence rates than using \TopK.

\vspace{1mm}
\item \textbf{High-performance implementation.} By aligning entry indices across nodes, \arctopK\ enables the use of index-free \texttt{All-Reduce} primitives. In contrast, the absence of shared  indices in \TopK\ necessitates slower \texttt{Gather/Scatter} operations, which increase latency and underutilize bandwidth. Our implementation demonstrates that \arctopK\ matches the accuracy of \TopK\ while reducing wall-clock training time by up to 60.69\%.
\end{itemize}

Table~\ref{table:communication_cmp} compares \RandK, \TopK, and \arctopK. The results show that \arctopK\ enables communication-efficient algorithms to achieve the same theoretical convergence rate and communication overhead as \RandK, while matching the empirical performance of \TopK. All reported convergence rates are obtained in combination with EF21M~\cite{fatkhullin2023momentum}, and the transient complexity reflects the number of iterations required for the asymptotic rate term to dominate. In the table, $N$ is the number of nodes, $T$ is the number of iterations, $\sigma$ is the standard deviation of the stochastic gradient noise. Tensors in all algorithm are compressed in 2-dimensional view of $m \times n$ while $K$ represents the number of selected rows in each tensor.

\vspace{0mm}
\noindent \textbf{Notation.}
We introduce the set $[N] := \{1,\cdots, N\}$. Given \(\bm A\in\mathbb{R}^{m\times n}\), operator \(\mathsf{vec}(\cdot)\) returns column-wise vectorization \(\mathsf{vec}(\bm A)\in\mathbb{R}^{mn}\) stacking the rows of \(\bm A\). \(\mathsf{diag}(\cdot)\) returns the vector of diagonal entries \(\mathsf{diag}(\bm M)\in\mathbb{R}^{n}\) for a square matrix \(\bm M\in\mathbb{R}^{n\times n}\). \(\mathsf{arg\,top}_K(\bm z)\) returns the index set of the \(K\) largest entries of \(\bm z\in\mathbb{R}^n\). Given an index set \(\mathcal{I}\) and a matrix \(\bm U\), \([\bm U]_{\mathcal{I},:} \) selects the rows of \(\bm U\) indexed by \(\mathcal{I}\) and zero other rows.

\vspace{0mm}

\begin{figure*}[tb!]
\centering
{\includegraphics[width=5.5in]{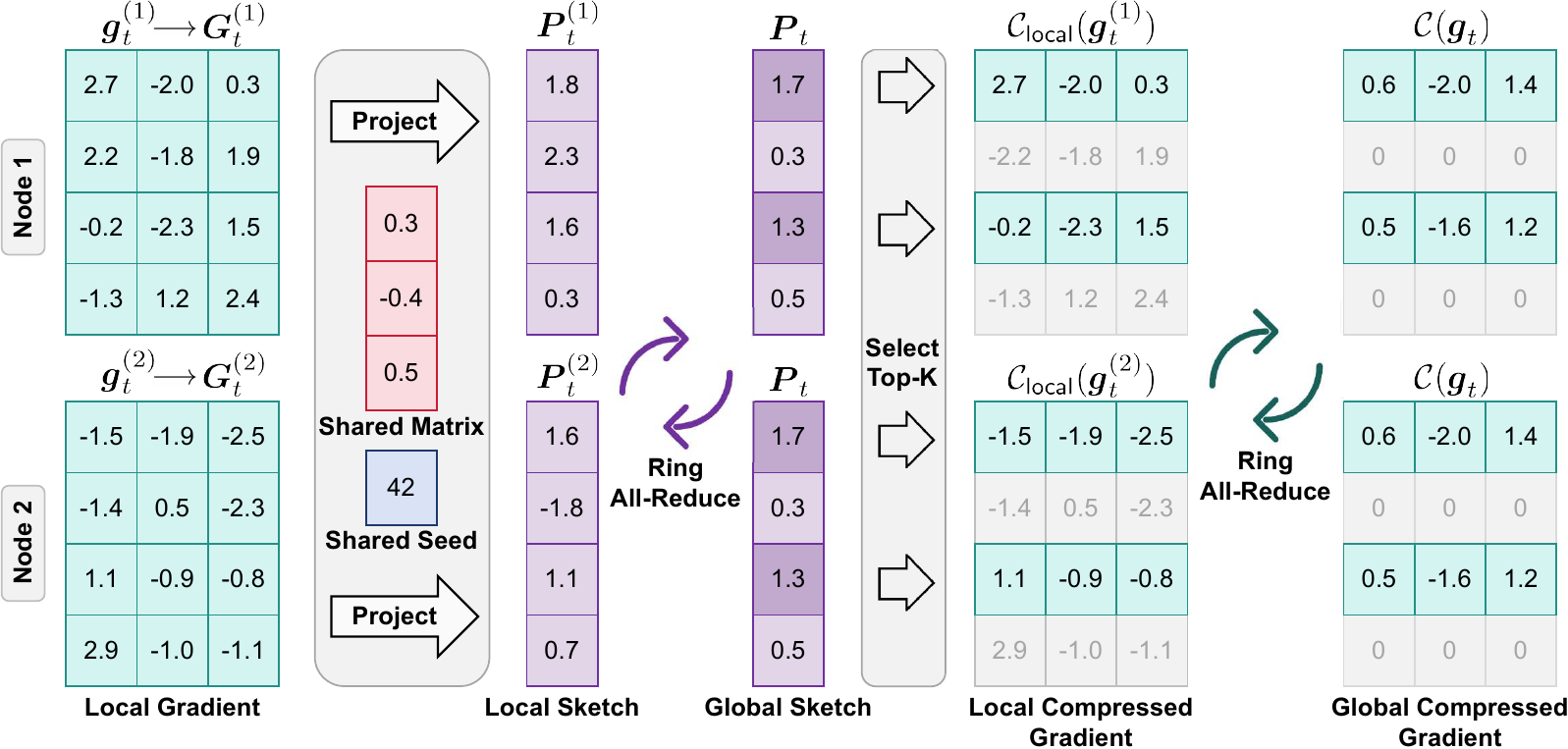}}\vspace{-2mm}
\caption{Workflow of the ARC-Top-$K$ algorithm, detailing the process of gradient compression and aggregation across two nodes in a distributed system.}
\label{fig:arctopk}
\vspace{-5mm}
\end{figure*}

\section{Related Works}

\noindent \textbf{Distributed Learning.} 
Distributed optimization underpins large-scale training in both data-center clusters and cross-device networks~\cite{shi2016edge}~\cite{satyanarayanan2017emergence}~\cite{chilimbi2014project,smola2010architecture,li2014communication,zhou2019edge}. In the \texttt{parameter server} architecture, a central server aggregates gradients from $N$ nodes and broadcasts the averaged update. Federated learning (FL) adapts this worker-server template under privacy and availability constraints, which exacerbate client drift and statistical heterogeneity~\cite{konevcny2016federated}~\cite{mcmahan2017communication}~\cite{kairouz2021advances}. Techniques such as client subsampling, proximal regularization, and adaptive aggregation have been proposed to mitigate these issues~\cite{li2020federated}~\cite{karimireddy2020scaffold}~\cite{wang2020tackling}. By contrast, synchronous data-parallel training in data centers employs collective communication (e.g., \texttt{All-Reduce}) to compute global averages without a central server, enabling the training of frontier large-scale models~\cite{brown2020language,raffel2020exploring,hoffmann2022training,touvron2023llama}. Our work focuses on the communication bottleneck common to both centralized and federated regimes.

\vspace{0mm}
\noindent \textbf{Communication Compression.}
Communication compression reduces data volume by transmitting compact representations of variables, thereby mitigating communication overhead and improving the scalability of distributed training. Existing approaches fall into three main categories: low-rank projection, quantization, and sparsification. Low-rank projection sketches gradients into lower-dimensional subspaces~\cite{vogels2019powersgd,zhao2024galore,chen2025greedy,he2024subspace}. Quantization encodes coordinates with fewer bits; stochastic schemes remain unbiased and integrate well with error feedback~\cite{alistarh2017qsgd}\cite{horvoth2022natural}. Sparsification transmits only a subset of entries. \RandK\  is unbiased but structure-agnostic, often discarding principal components and incurring large compression errors~\cite{wangni2018gradient}\cite{stich2018sparsified}. \TopK\ retains large-magnitude entries and typically yields better training performance, but it breaks \texttt{All-Reduce} compatibility and introduces additional overhead~\cite{lin2017deep}\cite{zhang2023evaluation}.

\vspace{0mm}
\noindent \textbf{Error Feedback.} Error feedback (EF) mitigates the adverse effects of compression by incorporating past residuals into subsequent gradient updates, thereby preserving more informative signals~\cite{seide20141}~\cite{karimireddy2019error}. EF21~\cite{richtarik2021ef21} extends this principle by maintaining a local gradient tracker for each node, which alleviates the impact of data heterogeneity and improves convergence rates. Building on this theoretical foundation, NEOLITHIC and its variants~\cite{huang2022lower}~\cite{he2023unbiased} establish lower bounds for distributed learning under communication compression. Notably, EF21M~\cite{fatkhullin2023momentum} achieves these bounds through momentum, attaining nearly optimal convergence rates. Our \arctopK\ compressor can be seamlessly integrated with EF21M to further enhance convergence in distributed settings.

\section{Limitations of the Top-$K$ Compressor} \label{sec:preliminary}

\vspace{1mm}
\noindent \textbf{Contractive Compressor.} 
Contractive compressors are widely used in communication-efficient distributed optimization. Their key property is that the compression error diminishes in expectation as the underlying variable approaches zero. A formal definition of a contractive compressor is given below:
\begin{definition}[Contractive Compressor]
\label{compressor-assumption}
A compressor $\mathcal{C}(\cdot)$ is contractive if and only if
\begin{align}
    \mathbb{E}_{\mathcal{C}}\Big[ \|\mathcal{C}(\bm{g})-\bm{g}\|_2^2 \Big]\leq (1-\alpha) \|\bm{g}\|_2^2,\quad \forall \bm{g}\in\mathbb{R}^{d},\label{eq:contraction}
\end{align}
where $0 <\alpha \le 1$ is the contractive factor. The expectation is taken over the randomness of the operator $\mathcal{C}$.

\end{definition}

\noindent\textbf{Non-Contraction of \TopK\ in Distributed Learning.} 
Although the classical \TopK\ operator is contractive when applied to a single vector, this property does not extend to distributed learning where the operator is applied independently across multiple nodes and the results are subsequently averaged. Formally, let $\mathcal{C}_\mathsf{local}(\cdot)$ denote the local \TopK\ compressor, ${\bm g} \coloneqq (1/N)\sum_{i=1}^N \bm g^{(i)}$ the uncompressed average, and ${\mathcal{C}}({\bm g}) \coloneqq (1/N)\sum_{i=1}^N \mathcal{C}_\mathsf{local}(\bm g^{(i)})$ the average of compressed vectors. The following proposition demonstrates that ${\mathcal{C}}({\bm g})$ may fail to be contractive in terms of the globally averaged gradient ${\bm g}$. This non-contractive property leads to unfavorable worst-case behavior, preventing algorithms using \TopK\ from achieving convergence rates as fast as those of \RandK\ .
\begin{center}
\setlength\fboxsep{4pt}
\colorbox{gray!20}{\parbox{\dimexpr\linewidth-2\fboxsep\relax}{
\begin{proposition}
\label{prop-non-contractive}
Let $N=2$, $d=2$, and $K=1$. Consider $\bm g^{(1)}=[-1,\,0.1]^{\top}, 
\bm g^{(2)}=[1,\,0.1]^{\top},$ and $\bm g = \tfrac{1}{2}(\bm g^{(1)} + \bm g^{(2)})$. Let $\mathcal{C}(\cdot)$ be \TopK\ compressor defined above. It holds that
\begin{align*}
\|{\mathcal{C}}({\bm g})-{\bm g}\|_2^2=\|{\bm g}\|_2^2,
\end{align*} 
which implies that ${\mathcal{C}}({\bm g})$ is a non-contractive compressor.
\end{proposition}
}}
\end{center}
\begin{proof}
\hspace{-0.5mm}Applying $\mathcal{C}_\mathsf{local}$ locally with $K{=}1$ yields \(\mathcal{C}_\mathsf{local}\left(\bm g^{(1)}\right)
=\text{\hspace{-1mm}}[-1,\,0]^{\top},\ \mathcal{C}_\mathsf{local}(\bm g^{(2)})=[1,\,0]^{\top}.
\)
Thus
\(
{\mathcal{C}}(\bm g)=\tfrac{1}{2}\bigl(\mathcal{C}_\mathsf{local}(\bm g^{(1)})+\mathcal{C}_\mathsf{local}(\bm g^{(2)})\bigr)=[0,\,0]^{\top},
\)
whereas
\(
{\bm g}=\tfrac{1}{2}(\bm g^{(1)}+\bm g^{(2)})=[0,\,0.1]^{\top}.
\)
Thus, it holds
\(
\|{\mathcal{C}}({\bm g})-{\bm g}\|_2^2=\|{\bm g}\|_2^2.
\)
\end{proof}

\vspace{-1mm}
\noindent \textbf{System-Level Limitations of \TopK.} \TopK\ also imposes significant communication overhead. 
Because the selected supports differ across nodes, indices must be transmitted together with values, 
increasing communication sizes and necessitating gather-and-merge operations. 
In parameter--server systems, aggregation ${\bm g} \coloneqq (1/N)\sum_{i=1}^N \bm g^{(i)}$ typically densifies the global gradient, requiring full-size communication from the server to all nodes. 
In collective-based clusters, heterogeneous supports preclude bandwidth-optimal \texttt{All-Reduce}, 
forcing implementations to fall back to variable-length \texttt{All-Gather} followed by a merge step, 
which underutilizes high-throughput collectives and increases memory traffic.

\section{All-Reduce Compatible Top\texorpdfstring{$K$}{K}}

\noindent \textbf{Compressor Design.} 
The key idea of \arctopK\ is to align sparsity patterns across all nodes while retaining the most significant entries of the compressed vector. To achieve this, \arctopK\ first reshapes the gradient vector into a matrix:
\begin{align}
\label{z72habsd00111}
\bm g_t^{(i)} \in \mathbb{R}^d \;\;\longrightarrow\;\; \bm G_t^{(i)} \in \mathbb{R}^{m \times n}, \quad n = d/m.
\end{align}
At iteration $t$, all $N$ nodes synchronize a random seed $s$, which deterministically generates a shared Gaussian projection matrix $
\bm V \in \mathbb{R}^{n \times r}, \text{with } \mathrm{vec}(\bm V) \sim \mathcal{N}(\bm 0, \bm I_{nr})$. Next, a row-wise sketch of the global gradient can be achieved through: 
\[
\bm P_t^{(i)} = \frac{1}{\sqrt{r}} \bm G_t^{(i)} \bm V \in \mathbb{R}^{m \times r},  \quad \bm P_t = \frac{1}{N} \sum_{i=1}^N \bm P_t^{(i)}.
\]
Row importance is then estimated as
\begin{align}
\label{zn28373}
\bm \Sigma_t = \mathrm{diag}(\bm P_t \bm P_t^\top) \in \mathbb{R}^m, \quad \mathcal{I}_t = \mathrm{arg\,top}_K(\bm \Sigma_t),
\end{align}
and $\mathcal{I}_t$ is the index set of the top-$K$ important rows. Next we aggregate these important rows from each node 
\begin{align}
\label{2zn20}
[\bm G_t]_{\mathcal{I}_t,:} = \frac{1}{N} \sum_{i=1}^N [\bm G_t^{(i)}]_{\mathcal{I}_t,:} \in \mathbb{R}^{m\times n}.
\end{align}

All unselected rows are set to zero in $\bm G_t$, and the result is reshaped back into $\hat{\bm g}_t = \mathrm{vec}(\bm G_t) \in \mathbb{R}^d$. Such $\hat{\bm g}_t$ is regarded as a compressed estimate of the globally averaged gradient ${\bm g}_t = (1/N)\sum_{i=1}^N \bm g_t^{(i)}$ using \arctopK, see Algorithm~\ref{alg:ar_topk} for more implementation details. 

\vspace{1mm}
\noindent \textbf{\arctopK\ Captures Important Entries.} Our key idea is that each $p$-th element of $\bm \Sigma_t \in \mathbb{R}^m$ provides an estimate of the importance of the $p$-th row of the global gradient matrix
$
\bm G_t = \tfrac{1}{N}\sum_{i=1}^N \bm G_t^{(i)}
$
in expectation. To see it, let $\bm u^{(i)} \in \mathbb{R}^{1\times n}$ be the $p$-th row of $\bm G_t^{(i)}$ and define $\bm u = \tfrac{1}{N}\sum_{i=1}^N \bm u^{(i)}$:
\vspace{-2mm}
\begin{align}
\label{z72ena}
\mathbb{E}[\bm \Sigma_{t,p}]
= \frac{1}{r}\mathbb{E}\!\left[\|\bm u \bm V\|_2^2\right]
= \frac{1}{r}\sum_{j=1}^r \mathbb{E}\!\left[(\bm u \bm v_j)^2\right]
= \|\bm u\|_2^2,
\end{align}
where $\bm \Sigma_{t,p}$ is the $p$-the element in $\bm \Sigma_t$. The expectation is taken over the random projection matrix $\bm V = [\bm v_1, \dots, \bm v_r]\in \mathbb{R}^{n \times r}$, whose columns are sampled i.i.d.\ from $\mathcal{N}(0,I_n)$. Equality~\eqref{z72ena} shows that $\bm \Sigma_{t,p}$ provides an estimate of the squared norm of the $p$-th row of $\bm G_t$, while the hyperparameter $r$ controls the variance of this estimate, with larger values of $r$ yielding smaller estimate variance. For this reason, expressions~\eqref{zn28373}--\eqref{z72ena} imply that \arctopK\ preserves the most significant rows while zeroing out the less significant ones.

\begin{algorithm}[tb!]
    \caption{\(\mathsf{\bf ARC}\mbox{-}\mathsf{\bf Top}\mbox{-}\bm{K}\big(\{\bm g_t^{(i)}\}_{i=1}^N,\, r,\, m,\, \mu\big)\).}
    \label{alg:ar_topk}
    \vspace{1pt}
    \textbf{Input}: $N$ nodes; projection dimension $r$; number of rows $m$; compressor ratio $\mu$ with selected row number  $K = \lceil \mu m \rceil$; local variable \(\bm g_t^{(i)}\in\mathbb{R}^{d}\) at iteration $t$ for \(i\in[N]\).\\
    \textbf{Output}: Compressed local gradient $\mathcal{C}_\mathsf{local}(\bm{g}_t^{(i)})$ and global gradient $\mathcal{C}(\bm g_t)$ with \(\bm{g}_t \coloneqq \sum_{i=1}^N\bm{g}_t^{(i)}\).\\
    \vspace{1mm}
    \textbf{(On $i$-th node)}\\
    Sample random seed \(s\) and synchronize it across nodes.\\
    Reshape \hspace{-0.5mm}\(\bm{G}_t^{(i)} \in \mathbb{R}^{m\times n}\ \text{s.t.}\ \mathsf{vec}(\bm{G}_t^{(i)})=\bm{g}_t^{(i)},\, n=d/m.\) \\
    \(\text{Generate } \bm V\in\mathbb{R}^{n\times r} \text{ with } s \text{ s.t. } \mathsf{vec}(\bm V)\sim\mathcal{N}(\bm 0,\bm I_{nr})\).\\
    \(\bm P_t^{(i)} \gets \bm G_t^{(i)}\,\bm V \in \mathbb{R}^{m\times r}\).\\
    \vspace{0.5mm}\(\bm P_t \gets \frac{1}{N}\sum_{i=1}^{N}\bm P_t^{(i)}\). \hfill (\texttt{All-Reduce})\\
    \vspace{0.5mm}\( \bm{\Sigma}_t \gets \mathsf{diag}\!\bigl(\bm{P}_t\bm{P}_t^{\top}\bigr) \in \mathbb{R}^m,\ \mathcal{I}_t \gets \mathsf{arg\, top}_K(\bm\Sigma_t)\in\mathbb{Z}^{K}\). \vspace{0.5mm}\\
    \(\mathcal{C}_\mathsf{local}(\bm{g}_t^{(i)}) \coloneqq \mathsf{vec}([\bm G_t^{(i)}]_{\mathcal{I}_t,:})\).\\
    \(\mathcal{C}(\bm{g}_t) \coloneqq \frac{1}{N} \sum_{i=1}^N\mathcal{C}_\mathsf{local}(\bm{g}_t^{(i)})\). \hfill (\texttt{All-Reduce})\\
    \Return $\mathcal{C}({\bm g}_t)$ and $\mathcal{C}_\mathsf{local}(\bm{g}_t^{(i)})$ for $i\in[N]$
\end{algorithm}

\vspace{1mm}
\noindent \textbf{\arctopK\ is a Contractive Compressor.} As shown in Proposition~\ref{prop-non-contractive}, the \TopK\ compressor is not contractive in distributed learning with respect to the globally averaged gradient. In contrast, the following proposition establishes that \arctopK\ is globally contractive, providing the foundation for its superior convergence properties.

\vspace{-1mm}
\begin{center}\setlength\fboxsep{1pt}\colorbox{gray!20}{\parbox{\dimexpr\linewidth-2\fboxsep\relax}{
\begin{proposition}
Let \({\bm g}_t \coloneqq \frac{1}{N}\sum_{i=1}^N{\bm g}_t^{(i)}\) be the global variable and let \(\mathcal{C}(\bm g_t)\) be the output of Algorithm~\ref{alg:ar_topk}. Then
\begin{align}
\mathbb{E}_\mathcal{C}\!\left[\|{\mathcal{C}}({\bm g}_t)-{\bm g}_t\|_2^2 \right]\le(1-\alpha)\,\|{\bm g}_t\|_2^2,\quad \alpha=\tfrac{K}{m},\nonumber
\end{align}
where the expectation is taken over the randomness of \(\mathcal{C}\).
\end{proposition}}}\end{center}
\begin{proof}
By the reshaping definition of \(\bm G_t\) in \eqref{z72habsd00111}, we have
\begin{align}
\|{\bm g}_t\|_2^2 &= \|{\bm G}_t\|_F^2, \label{eq:contract_proof_1}\\
\|{\mathcal{C}}({\bm g}_t)-{\bm g}_t\|_2^2 &= \|[\bm G_t]_{\mathcal{I}_t,:}-{\bm G}_t\|_F^2.\label{eq:contract_proof_2}
\end{align}
Let \(p_j \!\coloneqq\! \mathbb{P}[j \in \mathcal{I}_t]\) be the selection probability of the \(j\)-th {row}. According to the Lemma 1 in \cite{chen2025greedy},
\begin{align}
\| {\bm G}_t[l,:] \|_2^2 \le \| {\bm G}_t[j,:] \|_2^2 \Rightarrow p_l \le p_j,\ \quad \forall l,j \in [m].
\end{align}
Then by Chebyshev inequality, it holds
\begin{align}
&\mathbb{E}_\mathcal{C}\!\left[ \left\| [{\bm G}_t]_{\mathcal{I}_t,:} - {\bm G}_t \right\|_F^2 \right] 
= \| {\bm G}_t \|_F^2 - \sum_{j=1}^m p_j\, \| [{\bm G}_t]_{j,:} \|_2^2 \nonumber\\
\le& \| {\bm G}_t \|_F^2 - \frac{1}{m}\!\left(\sum_{j=1}^m p_j\right)\!\left(\sum_{j=1}^m \| [{\bm G}_t]_{j,:} \|_2^2\right)\nonumber\\
=& \| {\bm G}_t \|_F^2 - \frac{K}{m}\, \| {\bm G}_t \|_F^2.\label{eq:contract_proof_3}\end{align}

\noindent Substituting \eqref{eq:contract_proof_1} and \eqref{eq:contract_proof_2} into \eqref{eq:contract_proof_3} completes the proof.
\end{proof}

\noindent \textbf{\arctopK\ can be Efficiently Implemented.} By aligning sparsity patterns across nodes, \arctopK\ eliminates the need to transmit indices and enables the use of bandwidth-optimal \texttt{All-Reduce} rather than the more costly \texttt{All-Gather} (see Algorithm~\ref{alg:ar_topk}). This design achieves both high-fidelity gradient preservation and low communication cost, making \arctopK\ particularly effective for large-scale distributed learning. Moreover, in deep learning tasks where model weights and gradients are naturally stored in matrix form, the vectorization step required by \arctopK\ can be omitted, further simplifying its implementation.

\vspace{1mm}
\noindent \textbf{Communication Cost of \arctopK.} We analyze the communication overhead of row-sparsifying updates for an $m \times n$ parameter block replicated across $N$ nodes. At iteration $t$, each node selects $K$ rows, yielding sparsity $\mu \coloneqq K/m$, and transmits compressed updates through a collective primitive. In \arctopK, a compact rank-$r$ sketch is additionally exchanged to \emph{align} the selected row sets across nodes. ``Communication per iteration'' refers to the number of scalar entries transmitted per node per iteration. As summarized in Table~\ref{table:communication_cmp}, the dense baseline requires $2mn$ entries using \texttt{All-Reduce}. Independent \TopK\ instead relies on \texttt{All-Gather} to merge both values and indices, which incurs $(N-1)(nK + K)$ entries and fails to preserve global contractivity. By contrast, \RandK\ achieves global contractivity in expectation while communicating $2Kn$ entries via \texttt{All-Reduce}. Finally, \arctopK\ restores global contractivity and maintains index-free \texttt{All-Reduce} through the shared sketch, with total communication volume $2Kn + 2mr$, which further reduces to $2Kn + 2m$ when $r=1$.

\section{Communication-Efficient Method using ARC-Top-$K$}
\label{sec:arctopk}
The state-of-the-art communication-efficient method to solve problem \eqref{prop-non-contractive} is \textbf{\underline{M}omentum SGD with \underline{E}rror \underline{F}eedback (EF21M)}~\cite{fatkhullin2023momentum}, which corrects compression bias by recycling residual information. Equipping EF21M with \arctopK, we acheive the following recursions: 
\begin{subequations}\label{eq:ef21m}
\begin{align}
\bm{h}_t^{(i)} &= (1-\eta)\,\bm{h}_{t-1}^{(i)} + \eta\nabla F_i(\bm{x}_t, \bm{\xi}_t^{(i)}), \label{eq:ef21m-1} \\
\bm{g}_t^{(i)} &= \bm{g}_{t-1}^{(i)} + \mathcal{C}_\mathsf{local}\!\left(\bm{h}_{t}^{(i)}-\bm{g}_{t-1}^{(i)}\right), \label{eq:ef21m-2} \\
\bm{x}_{t+1} &= \bm{x}_t - \frac{\gamma}{N}\sum_{i=1}^N \bm{g}_t^{(i)}, \label{eq:ef21m-3}
\end{align}
\end{subequations}
where compressor $\mathcal{C}_\mathsf{local}(\cdot)$ is realized through \arctopK. Step~\eqref{eq:ef21m-1} updates the local gradient tracker, which serves as a momentum term to smooth stochastic gradient estimates. Step~\eqref{eq:ef21m-2} performs error feedback by compressing only the difference between this tracker and the previously transmitted vector using \arctopK, thereby retaining critical gradient information while reducing communication overhead. Because all nodes employ the same synchronized sketch, the resulting sparsity patterns are aligned across nodes, enabling efficient index-free \texttt{All-Reduce} aggregation in step \eqref{eq:ef21m-3}. 

Vanilla EF21M~\cite{fatkhullin2023momentum} employs \TopK\ as the compressor in step~\eqref{eq:ef21m-2}. Since \arctopK\ preserves the contraction property, we next establish that replacing \TopK\ with \arctopK\ enables EF21M to achieve faster theoretical convergence rates than its vanilla counterpart.

\begin{assumption}[Smoothness and lower boundedness]\label{asp:LL}
We assume  $f$ is $L$-smooth, $\|\nabla f(\bm{x}) - \nabla f(\bm{y})\| \le L \|\bm{x} - \bm{y}\|,\ \forall\bm{x},\bm{y} \in \mathbb{R}^d$.
Moreover, we assume that $f$ is lower bounded, i.e.,
$f^* := \inf_{\bm{x} \in \mathbb{R}^d} f(\bm{x}) > -\infty$.
\end{assumption}

\begin{assumption}[Stochastic gradient oracle]\label{asp:SGO}
For each $i \in [N]$, the stochastic gradient oracle $\nabla F_i(\bm{x}; \bm{\xi})$ is unbiased and has bounded variance, i.e.,
$\mathbb{E}_{\bm{\xi} \sim \mathcal{D}_i}[\nabla F_i(\bm{x}; \bm{\xi})] = \nabla f_i(\bm{x})$,
$\mathbb{E}_{\bm{\xi} \sim \mathcal{D}_i}\big[ \|\nabla F_i(\bm{x}; \bm{\xi}) - \nabla f_i(\bm{x})\|^2 \big] \le \sigma^2$.
\end{assumption}
\noindent We need several lemmas to facilitate the convergence analysis. 
\begin{lemma}\label{lm:lemma1}
Suppose Assumption~\ref{asp:LL} holds, and let the iterate be updated as 
\(\bm{x}_{t+1} = \bm{x}_{t} - \gamma \bm{g}_{t}\) 
for some vector \(\bm{g}_{t} \in \mathbb{R}^{d}\) and step size \(\gamma > 0\). Then the following inequality holds:
\begin{align}
f(\bm{x}_{t+1}) 
&\le f(\bm{x}_{t}) 
   - \frac{\gamma}{2}\|\nabla f(\bm{x}_{t})\|^{2} 
   - \frac{1}{4\gamma}\|\bm{x}_{t+1} - \bm{x}_{t}\|^{2} \nonumber\\
&\quad + \frac{\gamma}{2}\|\bm{g}_{t} - \nabla f(\bm{x}_{t})\|^{2}.
\end{align}
\end{lemma}
\begin{lemma}\label{lm:lemma2}
Suppose Assumption~\ref{asp:LL} holds, and let $\mathcal{C}$ be a contractive compressor with parameter $\alpha \le \tfrac{1}{2}$. Let the averaged sequence be
$\bm{h}_t \coloneqq \frac{1}{N}\sum_{i=1}^N \bm{h}_t^{(i)}$, $\bm{g}_t \coloneqq \frac{1}{N}\sum_{i=1}^N \bm{g}_t^{(i)}$ and $\nabla F(\bm{x}_{t}, \bm{\xi}_{t}) \coloneqq \frac{1}{N}\sum_{i=1}^N\nabla F_i(\bm{x}_{t}, \bm{\xi}_{t}^{(i)})$.Consider
\begin{align*}
    \bm{h}_{t} &= \bm{h}_{t-1} + \eta \bigl(\nabla F(\bm{x}_{t}, \bm{\xi}_{t}) - \bm{h}_{t-1} \bigr), \\
    \bm{g}_{t} &= \bm{g}_{t-1} + \mathcal{C}\bigl(\bm{h}_{t} - \bm{g}_{t-1}\bigr).
\end{align*}
Then the following inequality holds:
{
\begin{align}
\mathbb{E}\!\left[\|\bm{g}_t - \bm{h}_t\|^2\right]
&\le \left(1 - \tfrac{\alpha}{2}\right) 
   \mathbb{E}\!\left[\|\bm{g}_{t-1} - \bm{h}_{t-1}\|^2\right]+ \tfrac{\eta^2 \sigma^2}{N}\nonumber  \\
&\quad + \frac{4\eta^2}{\alpha}\, \mathbb{E}\!\left[\|\bm{h}_{t-1} - \nabla f(\bm{x}_{t-1})\|^2\right] \nonumber \\
&\quad + \frac{4L^2 \eta^2}{\alpha}\, \mathbb{E}\!\left[\|\bm{x}_t - \bm{x}_{t-1}\|^2\right]. \label{eq:lemma2} 
\end{align}
}
\end{lemma}
\begin{lemma}\label{lm:lemma3}
Suppose Assumption~\ref{asp:LL} holds, and let $0 < \eta \le 1$. For each $i \in \{1,\ldots,N\}$, define the sequence $\{\bm{h}_t^{(i)}\}_{t \ge 0}$ by the recursion
$\bm{h}_t^{(i)} = \bm{h}_{t-1}^{(i)} + \eta \bigl(\nabla F_i(\bm{x}_t, \bm{\xi}_t^{(i)}) - \bm{h}_{t-1}^{(i)}\bigr).$

Then, for all $t \ge 0$, it holds that
\begin{align}
\mathbb{E}\!\left[\bigl\|\bm{h}_{t+1} - \nabla f(\bm{x}_{t+1})\bigr\|^2\right]
&\le (1-\eta)\, \mathbb{E}\!\left[\bigl\|\bm{h}_t - \nabla f(\bm{x}_t)\bigr\|^2\right] \nonumber \\
&\hspace{-2cm} + \frac{3L^2}{\eta}\, \mathbb{E}\!\left[\|\bm{x}_{t+1}-\bm{x}_t\|^2\right]
   + \frac{\eta^2 \sigma^2}{N}.
\end{align}
\end{lemma}
The proofs of Lemma~\ref{lm:lemma1} and Lemma~\ref{lm:lemma3} are provided in~\cite{fatkhullin2023momentum}. Lemma~\ref{lm:lemma2} is new; its analysis is straightforward, and we omit the details here due to space constraints. The following theorem establishes the convergence of recursion~\eqref{eq:ef21m-1}--\eqref{eq:ef21m-3}. 
\begin{center}
\setlength\fboxsep{1pt}
\colorbox{gray!20}{\parbox{\dimexpr\linewidth-2\fboxsep\relax}{
\begin{theorem}[\textbf{ARC-Top-$\bm K$} Convergence with EF21M]\label{thm:cov-msgd} Under Assumptions \ref{asp:LL}, \ref{asp:SGO} , 
{
with the learning rate $\gamma\le 1/(4L)$, momentum $\eta$ and initial batch size $B_{\text{init}}$,
}
\vspace{-2.6mm}
{\begin{align*}
    &\frac{1}{T}\sum_{t=0}^{T-1}\mathbb{E}[\|\nabla f(\bm{{x}}_t\|^2]\leq\\
    &\quad\hspace{-1mm}\hspace{-0.6mm}\hspace{-0.6mm}\hspace{-0.6mm}\hspace{-0.6mm}\hspace{-0.6mm}\hspace{-0.6mm}\hspace{-0.6mm}\hspace{-0.6mm}\mathcal{O}\hspace{-0.6mm}\left(\hspace{-0.6mm}\frac{\delta_0}{\gamma T} \hspace{-0.6mm}
+ \hspace{-0.6mm}\frac{\eta^3 \sigma^2}{N \alpha^2} \hspace{-0.6mm}
+ \hspace{-0.6mm}\frac{\eta^2 \sigma^2}{N \alpha} \hspace{-0.6mm}
+ \hspace{-0.6mm}\frac{\eta \sigma^2}{N} \hspace{-0.6mm}
+ \hspace{-0.6mm}\frac{\eta \sigma^2}{N \alpha^2 B_{\text{init}} T} \hspace{-0.6mm}
+\hspace{-0.6mm} \frac{\sigma^2}{N \eta B_{\text{init}} T} \hspace{-0.6mm}\hspace{-0.6mm}\right)\hspace{-0.6mm},\nonumber\end{align*}} \hspace{-1mm}where  $\delta_0:=f(\bm{x_0})-\inf_{\bm{x}}f(\bm{x})$.%
If we further choose $\gamma$, $\eta$ and $B_{\text{init}}$ properly,
\textbf{ARC-Top-$\bm K$} with EF21M converges as 
{
\begin{align*}
    &\frac{1}{T}\sum_{t=0}^{T-1}\mathbb{E}\!\left[\|\nabla f({\bm{x}}_t)\|^2\right]\leq \\
    & \mathcal{O}\left(\hspace{-0.6mm}\left(\hspace{-0.6mm}\frac{L\delta_0\,\sigma^{2}}{NT}\hspace{-0.6mm}\right)^{\hspace{-0.6mm}\frac{1}{2}}\hspace{-0.6mm}\hspace{-0.6mm}\hspace{-0.6mm}+\hspace{-0.6mm}\left(\frac{\hspace{-0.6mm}L\delta_0\,\sigma}{\alpha^{\frac{1}{2}} N^{\frac{1}{2}} T}\hspace{-0.6mm}\right)^{\hspace{-0.6mm}\frac{2}{3}}\hspace{-0.6mm}\hspace{-0.6mm}\hspace{-0.6mm}+\hspace{-0.6mm}\left(\frac{\hspace{-0.6mm}L\delta_0\,\sigma^{\frac{2}{3}}}{\alpha^{\frac{2}{3}} N^{\frac{1}{3}} T}\hspace{-0.6mm}\right)^{\hspace{-0.6mm}\frac{3}{4}}\hspace{-0.6mm}\hspace{-0.6mm}\hspace{-0.6mm}+
        \frac{\hspace{-0.6mm}L\delta_0}{\alpha T}
    \right). \nonumber
\end{align*}}\end{theorem}}}\end{center}

\begin{proof}[Proof of Theorem \ref{thm:cov-msgd}]
By applying Lemma \ref{lm:lemma1} and decompose the error between $\bm{g}_t$ and $\nabla f(\bm{x}_t)$ into two terms by $\left\| \bm{g}_{t} - \nabla f(\bm{x}_{t}) \right\|^{2}
\le 2 \left\| \bm{g}_{t} - \bm{h}_{t} \right\|^{2}
+ 2 \left\| \bm{h}_{t} - \nabla f(\bm{x}_{t}) \right\|^{2}$ \\ we obtain:
\begin{align}
& \frac{1}{T} \sum_{t=0}^{T-1} \mathbb{E}\left[\|\nabla f(\bm{x}_t)\|^2\right]\nonumber\\
&\le \frac{2\delta_0}{\gamma T}
\text{\hspace{-0.5mm}}+\text{\hspace{-0.5mm}}  \frac{2}{T} \sum_{t=0}^{T-1} V_t
+ \frac{2}{T} \sum_{t=0}^{T-1} P_t
- \frac{1}{2\gamma^2 T} \sum_{t=0}^{T-1} R_t.\label{eq:A5}
\end{align}
It is noted that the quantities $V_t$, $P_t$, and $R_t$ are defined as follows:
$V_t =  \mathbb{E}\left[\|\bm{g}_t - \bm{h}_t\|^2\right]$,
\vspace{0.5mm}
$P_t = \mathbb{E}\left[\|\bm{h}_t - \nabla f(\bm{x}_t)\|^2\right] $,
$R_t =  \mathbb{E}\left[\|\bm{x}_{t+1} - \bm{x}_t\|^2\right]$.
From Lemma \ref{lm:lemma2}, it follows that
{
\begin{align}
\text{\hspace{-2mm}}\frac{1}{T}\text{\hspace{-0.5mm}} \sum_{t=0}^{T-1} V_t\text{\hspace{-0.5mm}}\le\text{\hspace{-0.5mm}}
\frac{8\eta^2}{\alpha^2 T} \sum_{t=0}^{T-1} P_t
\text{\hspace{-0.5mm}}+\text{\hspace{-0.5mm}} \frac{8L^2 \eta^2}{\alpha^2 T} \text{\hspace{-0.5mm}}\sum_{t=0}^{T-1} R_t
\text{\hspace{-0.5mm}}+\text{\hspace{-0.5mm}} \frac{2\eta^2 \sigma^2}{N \alpha}
\text{\hspace{-0.5mm}}+\text{\hspace{-0.5mm}} \frac{2 V_0}{\alpha T}. \text{\hspace{-0.5mm}}\label{eq:A6}
\end{align}
}
\vspace{-2mm}
From Lemma \ref{lm:lemma3}, it follows that
\begin{align}
\frac{1}{T} \sum_{t=0}^{T-1} P_t
&\le
\frac{3L^2}{\eta^2} \cdot \frac{1}{T} \sum_{t=0}^{T-1} R_t\
+ \frac{\eta \sigma^2}{N}
+ \frac{1}{\eta T} P_0 .\label{eq:A7}
\end{align}
Substituting the results of \ref{eq:A6} and \ref{eq:A7} into \ref{eq:A5}, and by choosing the step size $\gamma$ appropriately , we obtain:
\begin{align}
\frac{1}{T} \sum_{t=0}^{T-1} \mathbb{E}\left[\|\nabla f(\bm{x}_t)\|^2\right]\nonumber
&\le \frac{2\delta_0}{\gamma T}
+ \left( \frac{16\eta^3}{\alpha^2} + \frac{4\eta^2}{\alpha} + 2\eta \right) \frac{\sigma^2}{N}\nonumber\\&\quad
+ \left( \frac{16\eta}{\alpha^2} \hspace{-0.6mm}+\hspace{-0.6mm} \frac{2}{\eta} \right) \frac{P_0}{T}\hspace{-0.6mm}
+ \hspace{-0.6mm}\frac{4 V_0}{\alpha T}.
\end{align}
given $P_0 = \mathbb{E}\left[\|\bm{h}_0 - \nabla f(\bm{x}_0)\|^2\right] \leq \frac{\sigma^2}{N B_{\text{init}}}$, $V_0=0$ we obtain
\vspace{-5mm}
{
\begin{align*}
    &\frac{1}{T}\sum_{t=0}^{T-1}\mathbb{E}[\|\nabla f(\bm{{x}}_t\|^2]\leq\\
    &\quad\mathcal{O}\hspace{-0.6mm}\left(\hspace{-0.6mm}\frac{\delta_0}{\gamma T} \hspace{-0.6mm}
+ \hspace{-0.6mm}\frac{\eta^3 \sigma^2}{N \alpha^2} \hspace{-0.6mm}
+ \hspace{-0.6mm}\frac{\eta^2 \sigma^2}{N \alpha} \hspace{-0.6mm}
+ \hspace{-0.6mm}\frac{\eta \sigma^2}{N} \hspace{-0.6mm}
+ \hspace{-0.6mm}\frac{\eta \sigma^2}{N \alpha^2 B_{\text{init}} T} \hspace{-0.6mm}
+\hspace{-0.6mm} \frac{\sigma^2}{N \eta B_{\text{init}} T} \hspace{-0.6mm}\hspace{-0.6mm}\right)\hspace{-0.6mm},\end{align*}\vspace{-3mm}}

To further simplify the convergence bound, we impose constraints such that 
\begin{align*}
\frac{\eta^3 \sigma^2}{N \alpha^2} 
&\leq \frac{L \delta_0}{\eta T},
\frac{\eta^2 \sigma^2}{N \alpha} 
\leq \frac{L \delta_0}{\eta T},
\frac{\eta \sigma^2}{N} 
\leq \frac{L \delta_0}{\eta T},
\frac{\eta \sigma^2}{N \alpha^2 B_{\text{init}} T} 
\leq \frac{L \delta_0}{\eta T}.
\end{align*}
\vspace{0mm}
and
{\vspace{0mm}
\begin{align}
\frac{\sigma \sqrt{L \delta_0}}{\alpha\sqrt{B_{\text{init}} N} T}
 \hspace{-0.6mm}\hspace{-0.6mm}\leq\hspace{-0.6mm}{
\max \hspace{-0.6mm}\left\{\hspace{-0.6mm}\hspace{-0.6mm}
\frac{L \delta_0}{\alpha T}\hspace{-0.6mm},\hspace{-0.6mm}\hspace{-0.6mm}
\left( \frac{L \delta_0 \sigma^{\frac{2}{3}}}{\alpha^{\frac{2}{3}} N^{\frac{1}{3}} T} \hspace{-0.6mm}\right)^{\hspace{-0.6mm}\frac{3}{4}}\hspace{-0.6mm}\hspace{-0.6mm}\hspace{-0.6mm},\hspace{-0.6mm}
\hspace{-0.6mm}\left(\hspace{-0.6mm} \frac{L \delta_0 \sigma}{\alpha^{\frac{1}{2}} N^{\frac{1}{2}} T} \hspace{-0.6mm}\right)^{\hspace{-0.6mm}\hspace{-0.6mm}\frac{2}{3}}\hspace{-0.6mm}\hspace{-0.6mm}\hspace{-0.6mm},\hspace{-0.6mm}\hspace{-0.6mm}
\left( \hspace{-0.6mm}\hspace{-0.6mm}\frac{L \delta_0 \sigma^2}{N T} \hspace{-0.6mm}\right)^{\hspace{-0.6mm}\hspace{-0.6mm}\frac{1}{2}}\hspace{-0.6mm}\hspace{-0.6mm}
\right\}}.\nonumber
\end{align}\vspace{-2mm}}
Putting all these together, we achieve the convergence rate
{
\begin{align*}
    &\frac{1}{T}\sum_{t=0}^{T-1}\mathbb{E}\!\left[\|\nabla f({\bm{x}}_t)\|^2\right]\leq \\
    & \mathcal{O}\left(\hspace{-0.6mm}\left(\hspace{-0.6mm}\frac{L\delta_0\,\sigma^{2}}{NT}\hspace{-0.6mm}\right)^{\hspace{-0.6mm}\frac{1}{2}}\hspace{-0.6mm}\hspace{-0.6mm}\hspace{-0.6mm}+\hspace{-0.6mm}\left(\frac{\hspace{-0.6mm}L\delta_0\,\sigma}{\alpha^{\frac{1}{2}} N^{\frac{1}{2}} T}\hspace{-0.6mm}\right)^{\hspace{-0.6mm}\frac{2}{3}}\hspace{-0.6mm}\hspace{-0.6mm}\hspace{-0.6mm}+\hspace{-0.6mm}\left(\frac{\hspace{-0.6mm}L\delta_0\,\sigma^{\frac{2}{3}}}{\alpha^{\frac{2}{3}} N^{\frac{1}{3}} T}\hspace{-0.6mm}\right)^{\hspace{-0.6mm}\frac{3}{4}}\hspace{-0.6mm}\hspace{-0.6mm}\hspace{-0.6mm}+
        \frac{\hspace{-0.6mm}L\delta_0}{\alpha T}
    \right), \nonumber
\end{align*}
which completes the proof.}
\end{proof}

\noindent \textbf{Asymptotic Rate and Transient Iterations.} According to Theorem~\ref{thm:cov-msgd}, the asymptotic convergence rate is dominated by $\mathcal{O}(1/\sqrt{NT})$ as $T \to \infty$, which matches the rate achieved by algorithms using other compressors such as \TopK\ and \RandK. However, to reach this asymptotic regime, the algorithm must run for a sufficiently large number of iterations, commonly referred to as the transient phase. Based on Theorem~\ref{thm:cov-msgd}, we establish that the recursion \eqref{eq:ef21m-1}--\eqref{eq:ef21m-3} with \arctopK\ has transient complexity on the order of $\mathcal{O}(N/\sigma^2)$. This is substantially smaller than the transient complexity of \TopK, which scales as $\mathcal{O}(N^3/\sigma^2)$ (see Table~\ref{table:communication_cmp}). The advantage arises from the contractive property of \arctopK, in contrast to the non-contraction of \TopK.

\section{Experiment}

We evaluate \arctopK\ on both vision and language tasks, and further measure wall-clock time and accuracy across varying node counts. Comparisons are made against uncompressed training, standard \TopK, and \RandK\ baselines. Gradient compression begins after 1000 iterations, following \cite{vogels2019powersgd}. For transformers, we compress only two-dimensional tensors, which account for the majority of gradients. For convolutional neural networks, we follow \cite{wang2018atomo} and reshape four-dimensional convolutional kernels into two-dimensional matrices for compression.

\subsection{Pre-Training on CIFAR}
We evaluate \arctopK\ by training ResNet on CIFAR-10\cite{he2016deep}. Each experiment is repeated 3 times with different random seeds. We report mean accuracy and standard deviation. ResNet-18 and ResNet-50 are trained for 200 epochs with a learning rate of $1\times 10^{-3}$, weight decay of $5\text{\hspace{-0.5mm}}\times\text{\hspace{-0.5mm}} 10^{-4}$ and local batch size of $16$. All algorithms use sparsity $\mu\text{\hspace{-0.5mm}}=\text{\hspace{-0.5mm}}0.2$. For \arctopK, we set projection dimension $r=4$. Table~\ref{table:cifar10_acc} reports the results, with \texttt{\textbf{Dense}} denoting the uncompressed baseline. The highest accuracy in each model–optimizer pair is highlighted in bold. Results are shown both without EF and with EF enabled.

Across all settings, \arctopK\ consistently surpasses \RandK\ in accuracy. Moreover, \arctopK\ attains accuracy comparable to \TopK\ and the uncompressed baseline while incurring lower communication cost. When combined with MSGD, \arctopK\ achieves the highest accuracy for both ResNet-18 and ResNet-50, demonstrating an effective balance between communication efficiency and model performance.

\begin{table}[tb!]
\renewcommand{\arraystretch}{1.5}
\begin{center}
\vspace{-2mm}
\caption{\small Evaluation accuracy (\%) of training ResNet on CIFAR-10 with mean $\pm$ standard deviation over 3 runs. \vspace{-5mm}}
\begin{tabular}{lcccc}
\toprule\addlinespace[-0.2ex]
\multicolumn{1}{c}{\multirow{2}{*}{\textbf{Method}}} & \multicolumn{2}{c}{\textbf{Adam}} & \multicolumn{2}{c}{\textbf{MSGD}} \\ \addlinespace[-0.6ex]
 \cmidrule(lr){2-3} \cmidrule(lr){4-5} \addlinespace[-0.6ex]
 & \textbf{ResNet-18} & \textbf{ResNet-50} & \textbf{ResNet-18} & \textbf{ResNet-50} \\ \addlinespace[-0.6ex] 
\midrule
\texttt{\textbf{Dense}} & 94.01$\pm$0.42 & 94.02$\pm$0.15 & 95.08$\pm$0.08 & 94.96$\pm$0.31 \\ \addlinespace[-0.6ex] 
\multicolumn{5}{c}{\middlehrulefill\ \textit{(without EF)}\ \middlehrulefill}\\[-0.0ex]
\TopK\  & 93.58$\pm$0.10 & 93.86$\pm$0.14 & 94.14$\pm$0.10 & 94.58$\pm$0.32 \\
\RandK\  & 92.54$\pm$0.09 & 92.83$\pm$0.17 & 92.27$\pm$0.27 & 91.63$\pm$0.39 \\
\arctopK\  & 93.00$\pm$0.12 & 93.85$\pm$0.10 & 93.79$\pm$0.09 & 94.39$\pm$0.06 \\ \addlinespace[-0.6ex] 
\multicolumn{5}{c}{\middlehrulefill\ \textit{(with EF)}\ \middlehrulefill}\\[-0.0ex]
\TopK\   & \textbf{93.98}$\pm$\textbf{0.05} &\textbf{94.20}$\pm$\textbf{0.14} & 94.95$\pm$0.14 & 94.80$\pm$0.20 \\
\RandK\  & 93.90$\pm$0.09 & 93.88$\pm$0.41 & 94.81$\pm$0.18 & 94.72$\pm$0.38 \\
\arctopK\   & 93.91$\pm$0.05 & 93.97$\pm$0.06 & \textbf{95.00}$\pm$\textbf{0.18} & \textbf{95.00}$\pm$\textbf{0.09} \\
\bottomrule
\end{tabular}
\label{table:cifar10_acc}
\end{center}
\vspace{0mm}
\end{table}

\begin{table}[tb!]
\renewcommand{\arraystretch}{1.5}
\begin{center}
\caption{\small Results of fine-tuning RoBERTa-base on GLUE.\vspace{-4.5mm}}
{
\tabcolsep=1pt
\begin{tabular}{lccccccccc}
\toprule
\textbf{Method} & \textbf{CoLA} & \textbf{SST-2} & \textbf{MRPC} & \textbf{STS-B} & \textbf{QQP} & \textbf{MNLI} & \textbf{QNLI} & \textbf{RTE} & \textbf{Average} \\
\midrule
\texttt{\textbf{Dense}} & 64.16 & 94.73 & 93.15 & 91.16 & 91.86 & 87.51 & 92.81 & 81.59 & 87.12 \\
\midrule
\TopK\  & 65.29 & 94.27 & 92.76 & \textbf{91.13} & \textbf{92.02} & 87.45 & 92.84 & 80.51 & 87.03 \\
\RandK\  & 62.26 & 94.61 & 92.47 & 90.93 & 91.88 & 87.34 & \textbf{93.04} & 80.51 & 86.63 \\
\arctopK\  & \textbf{66.04} & \textbf{94.84} & \textbf{93.22} & 91.03 & 91.98 & \textbf{87.58} & \textbf{93.04} & \textbf{81.59} & \textbf{87.42} \\
\bottomrule
\end{tabular}}
\label{table:glue_acc}
\end{center}
\vspace{-5mm}
\end{table}

\vspace{-1mm}
\subsection{Fine-Tuning on GLUE}\label{exp:glue}
\noindent We fine-tune RoBERTa-base on GLUE using 4 NVIDIA RTX~4090 (24\,GB) GPUs, employing Adam without weight decay and a maximum sequence length of 512\cite{liu2019roberta}. All tasks are fine-tuned for 30 epochs with a linear learning-rate schedule. We set sparsity $\mu{=}0.2$ for all compressors and fix the projection rank at $r{=}4$ for \arctopK. For CoLA and MRPC we use a learning rate of $3\times10^{-5}$ with local batch size 32, whereas all other tasks use $1\times10^{-5}$ with size 16. 

Table~\ref{table:glue_acc} reports accuracy on the 8 GLUE tasks, with the best score per column in \textbf{bold}. \texttt{\textbf{Dense}} denotes the baseline without gradient compression. Across tasks, \arctopK\ consistently surpasses \RandK, and it matches or exceeds \TopK\ and the uncompressed baseline while requiring substantially less communication, achieving the highest average score and a favorable accuracy--communication tradeoff.

\vspace{-1mm}
\subsection{Pre-Training on C4}\label{sec:pt_c4}
We pre-train LLaMA models on the C4 corpus~\cite{touvron2023llama}~\cite{raffel2020exploring} using standard Adam with a local batch size of 128 and a sequence length of 256. Following~\cite{zhao2024galore}, we train LLaMA-60M and LLaMA-130M on 1.1B and 2.2B tokens respectively. All compressors use error feedback, retain $\mu=0.2$ entries per tensor, and set $r{=}4$ for \arctopK. We select the best result from learning rates $\{1\times10^{-3},\,2\times10^{-3}\}$.

Table \ref{table:c4} reports the final validation perplexity. \arctopK\ outperforms \RandK\ in pre-training tasks consistently. \arctopK\ also outperform \TopK\ when training large models of same compression ratio with less communication.
\begin{table}[t!]  
\renewcommand{\arraystretch}{1.5}
\begin{center}
\caption{ \small  Validation perplexity of pre-training LLaMA on C4.\vspace{-2mm}}
{\tabcolsep=2pt
\begin{tabular}{lcc}
\toprule
\textbf{Method} & \textbf{LLaMA-60M} & \textbf{LLaMA-130M}   \\
\midrule
\texttt{\textbf{Dense}}       & 30.59& 24.72\\
\midrule
\RandK\     & 43.31& 36.31\\
\TopK\      & 32.20& \textbf{29.19}\\
\arctopK\  & \textbf{33.82}& 26.43\\
\bottomrule
\end{tabular}
}
\label{table:c4}
\end{center}
\vspace{-1mm}
\end{table}

\subsection{Further Study}
\noindent \textbf{Wall-clock performance}. We measure per-iteration wall-clock time for pre-training LLaMA models (60M, 130M, 350M, and 1B) on C4 using 4 NVIDIA A100 (40\,GB) GPUs. Each run uses Adam for 50 iterations. We report the mean time per iteration. The local batch size is fixed at 1. For sparsification, the sparsity is set to \(\mu{=}0.2\). Gradients are communicated via the NCCL backend using shared memory (SHM) transport, and NVLink peer-to-peer is explicitly disabled to emulate a low-bandwidth, multi-machine setting. As summarized in Table~\ref{table:c4_time}, sparsification yields limited time savings for small models because the compression overhead offsets communication gains. However, as model size increases and communication dominates, the relative overhead of compression diminishes and benefits become pronounced. For LLaMA-1B, \arctopK\ reduces the average per-iteration time from 3.0854\,s to 1.2130\,s with a \(60.69\%\) reduction, and achieves the best wall-clock performance across all model sizes.

\begin{table}[t!]  
\renewcommand{\arraystretch}{1.5}
\begin{center}
\caption{ \small Average per‐iteration training time (in seconds) of sparsification algorithms when training LLaMA.\vspace{-2mm}}
{\tabcolsep=2pt
\begin{tabular}{lcccc}
\toprule
\textbf{Method}             & \textbf{LLaMA-60M}    & \textbf{LLaMA-130M}   & \textbf{LLaMA-350M}   & \textbf{LLaMA-1B}      \\
\midrule
\texttt{\textbf{Dense}}      & 0.1469 & 0.3243 & 0.8775 & 3.0854 \\
\midrule
\TopK\   & 0.1526 & 0.3287 & 0.8497 & 2.9628 \\
\RandK\  & 0.0752 & 0.1480 & 0.3544 & 0.9581 \\
\arctopK\  & 0.1061 & 0.1798 & 0.3745 & 1.2130 \\
\bottomrule
\end{tabular}
}
\label{table:c4_time}
\vspace{-4mm}
\end{center}
\end{table}

\noindent \textbf{Scaling to 64 nodes}. To validate the scalability of \arctopK\ in multi-node distributed settings, we conduct experiments by fine-tuning RoBERTa-base under the description in Section~\ref{exp:glue}, while fixing the global batch size to 64. The number of nodes scales from 8 to 64 for each method. Each experiment is repeated 3 times with independent random seeds, and the results are reported as the mean and standard deviation in Table~\ref{table:nodes}. The results indicate that \arctopK\ consistently matches the performance of the dense baseline across scales from 8 to 64 nodes, while substantially outperforming the \RandK\ compressor under the same compression ratio. These findings underscore the strong scalability of \arctopK, highlighting its promise for large-scale distributed training and federated learning scenarios involving massive edge devices.

\noindent \textbf{Saving in communication bits}. We quantify communication bits by pre-training LLaMA-130M for $10{,}000$ iterations with a global batch size of $256$; all remaining hyperparameters follow Section~\ref{sec:pt_c4}. Figure~\ref{fig:loss_bit} plots training loss against cumulative communicated bits.In the figure, \arctopK\ reaches the target loss with fewer bits than the dense baseline. Furthermore, although \arctopK\ involves slightly more communication per iteration than \RandK, its faster convergence yields lowest {total} communication to attain the same loss.

\begin{table}[t!]  
\renewcommand{\arraystretch}{1.5}
\begin{center}
\caption{ \small Evaluation matthews correlation of fine-tuning pre-trained RoBERTa-base on COLA with mean $\pm$ standard deviation over 3 runs.\vspace{-2mm}}
{\tabcolsep=5pt
\begin{tabular}{lcccc}
\toprule
\textbf{Method} & \textbf{8 Nodes}    & \textbf{16 Nodes}   & \textbf{32 Nodes}   & \textbf{64 Nodes}      \\
\midrule
\texttt{\textbf{Dense}}       & 64.00$\pm$0.80 & 63.67$\pm$1.66 & 64.20$\pm$1.49 & 63.78$\pm$1.12 \\
\midrule
\TopK\     & 64.01$\pm$0.39 & 63.34$\pm$0.90 & 63.50$\pm$1.27 & \textbf{64.02}$\pm$\textbf{1.95} \\
\RandK\      & 62.67$\pm$1.48 & 62.60$\pm$1.04 & 62.13$\pm$1.15 & 62.42$\pm$0.16 \\
\arctopK\  & \textbf{64.50}$\pm$\textbf{1.30} & \textbf{64.15}$\pm$\textbf{0.47} & \textbf{64.53}$\pm$\textbf{0.64} & 63.70$\pm$2.25 \\

\bottomrule
\end{tabular}
}
\label{table:nodes}
\end{center}
\vspace{-2mm}
\end{table}

\section{Conclusion}
In this work, we propose \arctopK, an All-Reduce-compatible sparsification scheme for communication-efficient distributed learning. By aligning sparsity patterns across nodes via a lightweight sketch, \arctopK\ lowers communication while preserving convergence comparable to dense training. We establish convergence guarantees under distributed settings and experiments on RoBERTa fine-tuning and large-scale multi-node tasks show that \arctopK\ matches dense baselines while significantly outperforming random sparsification at the same compression ratio. Finally, \arctopK\ scales efficiently from 8 to 64 nodes, demonstrating its suitability for large-scale distributed and federated learning.

\begin{figure}[tb!]
\centering
{\includegraphics[width=2.5in]{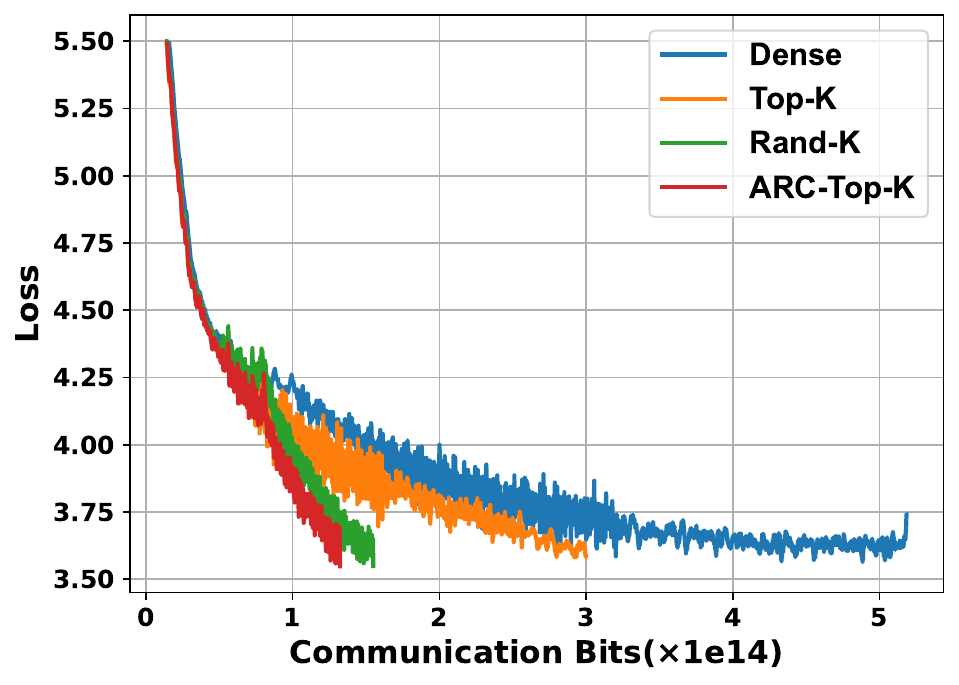}}
\vspace{-4mm}\caption{Loss curves of pre-training LLaMA-130M on C4.}
\label{fig:loss_bit}
\vspace{2mm}
\end{figure}

\vspace{-0.5mm}

\bibliographystyle{ieeetr}
{
\footnotesize
\bibliography{reference}
}

\end{document}